\documentclass[runningheads]{llncs}
\usepackage{graphicx}
\usepackage{tikz}
\usepackage{comment}
\usepackage{amsmath,amssymb} % define this before the line numbering.
\usepackage{color}

\usepackage[accsupp]{axessibility}  % Improves PDF readability for those with disabilities.

\usepackage{booktabs}

\usepackage{algorithm}
\usepackage{algorithmic}
\usepackage{multirow}
\usepackage{array, makecell}

\begin{document}
% \linenumbers
%\pagestyle{headings}
\mainmatter

\title{Adaptive Feature Interpolation for Low-Shot Image Generation} % Replace with your title

%\begin{comment}
% \titlerunning{ECCV-22 submission ID \ECCVSubNumber} 
% \authorrunning{ECCV-22 submission ID \ECCVSubNumber} 
% \author{Anonymous ECCV submission}
% \institute{Paper ID \ECCVSubNumber}
%\end{comment}
%******************

% \begin{comment}
\titlerunning{Adaptive Feature Interpolation}
% \author{
% Mengyu Dai\\
% Microsoft\\
% \and
% Haibin Hang\\
% Amazon\\
% {\tt\small \{dzld00511,haibin.hang312,xiaoyang.guo.fl\}@gmail.com}
% \and
% Xiaoyang Guo\\
% Meta\\
% }

\author{Mengyu Dai\inst{1} \and
Haibin Hang\inst{2} \and
Xiaoyang Guo\inst{3}}
\authorrunning{Dai et al.}

\institute{Salesforce\and Amazon\and Meta 
\\
\email{\tt\small
mdai@salesforce.com,
haibinh@amazon.com,
xiaoyangg@fb.com
}}
%\{dzld00511,haibin.hang312,xiaoyang.guo.fl\}@gmail.com}} 
%******************
\maketitle

%%%%%%%%% ABSTRACT
\begin{abstract}
Training of generative models especially Generative Adversarial Networks can easily diverge in low-data setting. To mitigate this issue, we propose a novel implicit data augmentation approach which facilitates stable training and synthesize high-quality samples without need of label information. 
Specifically, we view the discriminator as a metric embedding of the real data manifold, which offers proper distances between real data points.
We then utilize information in the feature space to develop a fully unsupervised and data-driven augmentation method. 
Experiments on few-shot generation tasks show the proposed method significantly improve results from strong baselines with hundreds of training samples.
\end{abstract}

%%%%%%%%% BODY TEXT
\section{Introduction}
Majority of learning algorithms today favor the feed of large training data. However, it is often difficult to collect sufficient amount of high-quality data for usage. 
In addition, 
%with the goal towards more powerful artificial intelligence,
intelligent systems like human brains do not need millions of samples to learn useful patterns and are energy-efficient. 
On the premise of it, learning with small data has been an important research area in various tasks \cite{Zhang_2021_CVPR,Hong_2021_CVPR,Li_2021_CVPR,Stojanov_2021_CVPR,Huang_2021_ICCV,Liu_2021_ICCV,Yao2020Automated,Tian2020RethinkingFI,Choe_2017_ICCV}. Among numerous promising works along the direction, a limited amount target on generative models. %\cite{Karras2020ada,zhao2020diffaugment,Zhang2020Consistency, lecamgan,liu2021}.
Training of generative models especially Generative Adversarial Networks (GANs) \cite{gan} can easily diverge in low-data setting. To overcome the issue, people come up with methods focusing on different aspects in GAN training, such as data augmentation \cite{Karras2020ada,zhao2020diffaugment}, network architecture design \cite{liu2021,Karras2019stylegan2} and applying regularization \cite{Zhang2020Consistency,lecamgan}.
Data augmentation can substantially increase the size of usable samples and enable stable training \cite{zhao2020diffaugment}.
%and better convergence\cite{zhao2020diffaugment}. 

Unlike above data augmentation approaches for generative models which target on image domain, we propose a simple yet effective method to implicitly augment training data without supervision. 
To our knowledge, it is the first attempt to interpolate the multidimensional output feature of the discriminator for data generation.  
This can possibly be due to the fact that applications using GAN frameworks usually adopt objectives with $1$-dimensional discriminator output, such as vanilla GAN \cite{gan} and Wasserstein GAN \cite{wgan}. 
%In \cite{dai2021awgan} the authors discussed how multiple critic output neurons may benefit the GAN training based on the understanding that high-dim feature vector is more informative.
Recently, 
%Liu et al. \cite{liu2021} propose a strong network architecture equipped with hinge loss using multidimensional discriminator output. 
%Their proposed SLE modules make it possible to generate high-resolution and visually appealing samples using only hundreds of samples.
Dai and Hang \cite{Dai_and_Hang_2021_ICCV} %incorporate a triplet network to a generative model
introduce a metric learning perspective to the
GAN discriminator with multidimensional output and reveal an interesting \textit{flattening effect}: along the training process, the learned metric gradually becomes more uniform and flat. 
The observation 
%with multiple critic outputs, especially the flattening effect, 
inspire us to explore the possibility of implementing augmentation in feature space in an unsupervised fashion. 
%Due to the high dimensional nature of image data, a recent paper \cite{} demonstrate both theoretically and empirically that interpolation almost surely never occurs in high-dimensional spaces (> 100) regardless of the underlying intrinsic dimension of the data manifold.

In this paper, we propose a way to implicitly augment training data by taking the advantage of the flattening effect of discriminators with multiple output neurons. 
An intuitive understanding %easy illustration 
is that, compared to highly sparse and nonlinear nature of real data manifold~\cite{r2021learning}, the low-dimensional feature space is relatively dense and flat. Hence applying interpolation in feature space yields a feasible way for augmentation with higher fidelity. An example of the effect of feature interpolation during training with StyleGAN2 architectures is shown in Figure~\ref{fig:interp_vs_no}. Both sessions utilized adaptive image augmentation \cite{Karras2020ada}, while feature interpolation significantly stabilized training. 
%makes more sense while data augmentation. 
%In its simplest form, 
%instead of updating neural network parameters by the gradients raised by each real image, we suggest updating the gradients by random linear combination of the gradients raised by some nearby real images. 
%Another novelty in this work is that we propose a simple metric to quantify diversity of generated samples, as a supplement to existing evaluation methods such as Fr\'echet Inception Distance (FID) and Inception Score (IS). Chong et al. \cite{Chong_2020_CVPR} shows 
%FID and IS are biased for a finite sample set,  and the bias term also depends on the particular model being evaluated. Under low-data setting, the amount of trustworthy information one can obtain from these metrics is even more limited.  
\vspace{-1em}
\begin{figure}[!ht]
    \centering
    \includegraphics[width=0.7\linewidth]{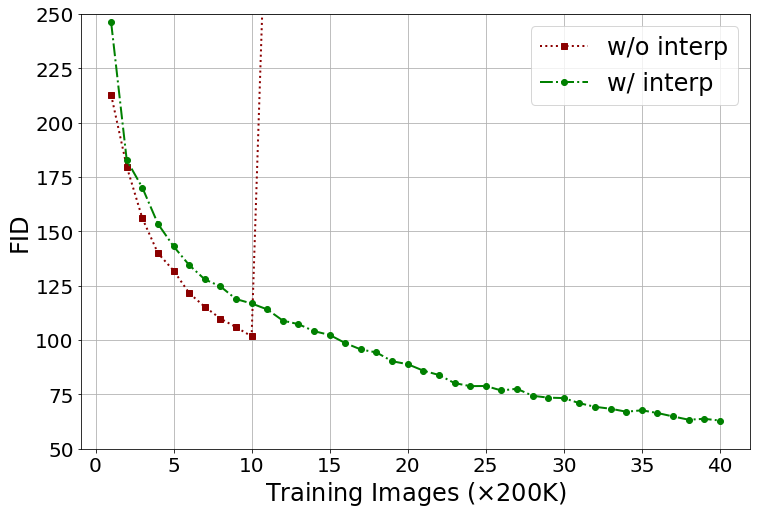}
    \caption{FID ($\downarrow$) during training on {\it Shells} dataset. Red: without feature interpolation; Green: with feature interpolation. Here ``training images'' refers to  training iteration × batch size.}
    \label{fig:interp_vs_no}
    \vspace{-1em}
\end{figure}
The novelty of this work is summarized as follows: 1. We propose an implicit data augmentation method in multidimensional feature space of discriminator output. To our knowledge, this is the first attempt in unsupervised image generation. 
2. We develop a data-driven approach for augmentation, with criteria based on the underlying structure of feature space during training.
%3. We bring up a simple diversity metric for image generation, which serves as an supplement to common metrics to help evaluating diversity of synthesized samples. 
3. Results from few-shot image generation experiments show significant improvements on several benchmark datasets. %Figure~\ref{fig:eigen_simclr} suggests potential usage of the method in other tasks.

%The proposed method can generate high-quality samples with fewer training images.

\section{Related Work}
%low-shot data generation; mixup; GAN from geometric perspective; embedding space interpolation

\subsection{Low-Shot Data Generation}
Recent works contribute to low-shot data generation from different perspectives.
%: data augmentation, regularization and network architecture design and so on.
%Currently implementations of data augmentation for generative models are mostly in the image domain. 
Karras {\it et al.} \cite{Karras2020ada} implemented augmentation on images with adaptive probabilities by using a validation set. 
Zhao {\it et al.} \cite{zhao2020diffaugment} proposed DiffAugment which applies random differentiable augmentation on both real and generated images.
The above methods significantly improve the amount of available training data in image domain thus remarkably prevent discriminator from overfitting. 
%On the line of regularization,
Tseng {\it et al.} \cite{lecamgan} designed a regularization loss term on predictions of discriminator by tracking the moving averages of discriminator predictions during training.
Zhang {\it et al.} \cite{Zhang2020Consistency} proposed consistency regularization for GANs, with the argument on the invariance of samples after transformation.
%For network architectures,
Liu {\it et al.} \cite{liu2021} proposed an SLE module and an encoder-decoder reconstruction regularization on discriminator to improve training stability in low-data training settings. 
%This design makes it possible to generate high-resolution and visually appealing samples using only hundreds of images.
For other directions, one can refer to \cite{ojha2021few-shot-gan,hong2020f2gan,hong2020deltagan,li2020few}.
Different from above techniques, our proposed method is implemented in the multi-dimensional feature space of discriminator output. 
It does not conflict with data augmentation techniques in the image domain and is independent of network architectures.
%, both of which can be implemented together during training. 

\subsection{Geometric Interpretations in GANs}
The key idea of this paper comes from the interesting flattening effect of discriminator observed in~\cite{Dai_and_Hang_2021_ICCV}. Particularly, in their paper~\cite{Dai_and_Hang_2021_ICCV}, Dai and Hang interpret the discriminator as a metric generator which learns some intrinsic metric of real data manifold such that the manifold is flat under the learned metric. 
%However, we believe this flattening effect should be shared by general GAN frameworks with multiple critic output neurons. 
In this paper, we observe the similar behaviors in the {\it geometric} GAN~\cite{lim2017geometric} framework with hinge loss.
Specifically, {\it geometric} GAN use SVM separating hyper-plane to maximizes the margin between real/fake feature vectors and use hinge loss for discriminator which is simple and fast.
%it proposed a generative framework using metric learning with geometric interpretation. They have shown under the framework real data manifold becomes more and more uniform during training.  
Similar effects are also mentioned in \cite{Shao_2018_CVPR_Workshops} by Shao {\it et al.} which shows manifolds learned by deep generative models are close to zero curvature.

\subsection{Interpolation in Feature Space and Mixup}
Combining features in the embedding space are shown to be helpful in image retrieval~\cite{chum2007total,chum2011total,turcot2009better,arandjelovic2012three}. 
Recently, Ko {\it et al.} \cite{ko2020embedding} proposed embedding expansion which utilized a combination of embeddings and performs hard negative mining to learn informative feature representations.
%In \cite{devries2017dataset}, the authors used nearest neighbors to manipulate feature vectors and use synthetic feature vector directly in supervised learning or feed it to a decoder to reconstruct a corresponding virtual data point.
DeVries and Taylor \cite{devries2017dataset} claimed that simple transformations to feature space results in plausible synthetic data due to manifold unfolding in feature space. 
%Instead of using any decoder to reconstruct a virtual data from the synthetic feature vector, we use Lemma\ref{lem:backpropagation} to approximate the gradient decent of the virtual data.  
Verma {\it et al.} \cite{verma2019manifold} introduced Manifold Mixup, which implemented interpolation between hidden feature vectors to obtain smoother decision
boundaries at multiple levels of representation. Furthermore, in \cite{verma2019manifold} authors also indicated the flattening effect of Manifold Mixup in learning representations.
%However, we observe that the the flattening effect exists more generally in feature embedding/extraction neuron networks.

%\subsection{Mixup}
Another branch of data augmentation techniques takes advantage of the prior knowledge in learning tasks to interpolate both training images and the corresponding labels. 
Zhang {\it et al.} \cite{zhang2018mixup} suggested that linear interpolations of data samples should lead to linear interpolations of their associated labels.
%introduced {\it mixup}, which linearly interpolated images and labels .
%In stead of mixing up pixel values of input images, Manifold Mixup\cite{verma2019manifold} mixes up hidden feature vectors in a supervised learning neuron network. 
Kim {\it et al.} \cite{kim2021comixup} applied batch mixup and formulated the optimal construction of a batch of mixup data.
Other works along the track also show improvements on various discriminative learning tasks \cite{kimICML20,verma2019manifold,Yun_2019_ICCV,NEURIPS2019_36ad8b5f}.
Despite the promising progress, applying these methods require augmentation in training data's associated labels, which does not suit the use case in this paper. %Moreover, directly applying interpolation on images could largely twist the real data distribution, thus is not ideal for the goal of a generative task aiming at learning distribution from data itself \cite{zhao2020diffaugment}. 
Different from above augmentation methods which are mostly used in supervised discriminative tasks, in this work we develop an implicit augmentation method using data-driven feature interpolation, which is suitable for generative tasks in a unsupervised fashion.

\section{Methodology}
In this section we introduce the main idea of the paper and the proposed implicit augmentation algorithm for low-shot generation in detail.
\begin{figure}
    \centering
    \includegraphics[width=0.8\linewidth]{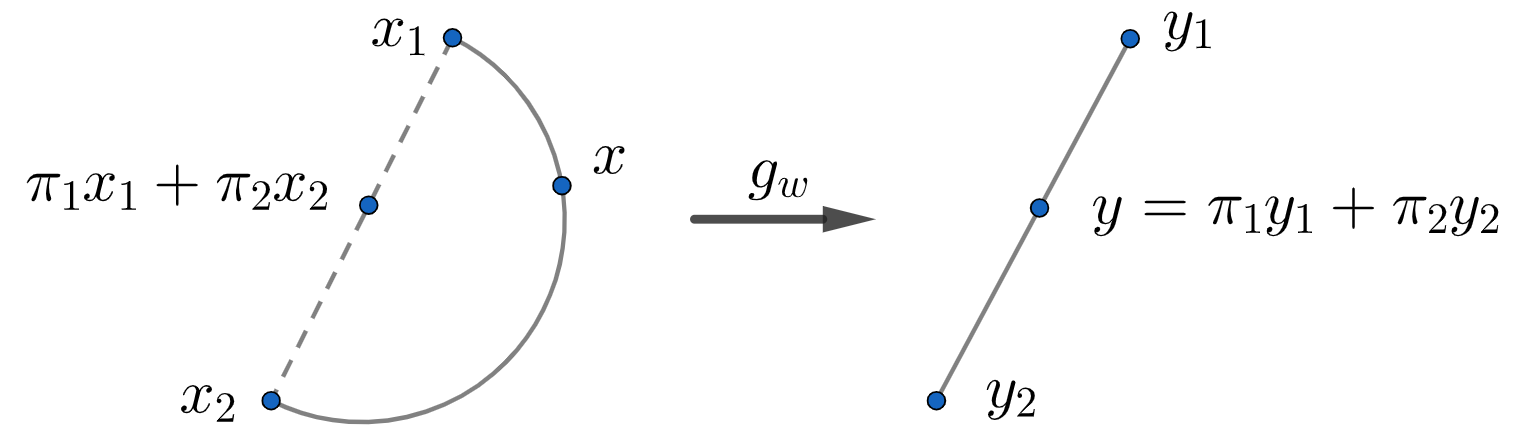}
    \caption{Direct interpolation of real data likely returns points far away from real data manifold; With flattening effect, we propose interpolation in feature space which returns feature $y$ as an approximation of using some imaginary ``real'' sample $x$.}
    \label{fig:interp}
    \vspace{-2em}
\end{figure}

\subsection{The Flattening Effect of Discriminator}
In this paper, we denote $y$ as a {\it valid} feature vector iff. $y=g_w(x)$ for some data $x$ from real data manifold given some deep metric learning network $g_w$. 
A simple illustration is shown in Figure~\ref{fig:interp}.
One question we are interested in is: 
How far away is the interpolation of a group of valid feature vectors from some individual valid feature vector?
In the following we will address this question in a GAN framework using observations from experimental results. 

We adopt the default training setting in \cite{liu2021} and {\it Shells} dataset which contains 64 diverse images for experimentation.
\cite{liu2021} uses hinge loss as GAN objective, thus the discriminator has metric learning effect and is equipped with multi-dimensional output. The hinge objective can be formulated as:

\begin{align}
    L_G = &\,\, \mathbb{E}_{z\sim \mathcal{N}}[-D(G(z))] \label{eqn:lg}\\
    L_D = &\,\, \mathbb{E}_{x\sim P_{real}}[\max(0,1-D(x))] \nonumber\\
          & + \mathbb{E}_{x\sim P_{fake}}[\max(0,1+D(x))]  \label{eqn:ld} \ ,
\end{align}
where $D(x)$ is also named $g_w(x)$ in this paper. 

We utilize the following way introduced in \cite{Dai_and_Hang_2021_ICCV,Shao_2018_CVPR_Workshops} to detect the change of the learned metric along the training process: 
(i) For each iteration $i$, sample some real data points to form a finite metric space $X_i$;
(ii) Construct the normalized distance matrix of $X_i$ under the learned metric; 
%and normalize it to $D_i$ after dividing by the diameter $\max_{x,x'\in X_i} d_i(x,x')$ of $X_i$;
(iii) Apply multidimensional scaling (MDS) to the normalized distance matrix to obtain the decreasing (finite) sequence of eigenvalues $C_i=\{\lambda^{(i)}_1,\lambda^{(i)}_2,\cdots,\lambda^{(i)}_b\}$, where $b$ is the number of sample points.
\begin{figure}[!ht]
    \centering
    \includegraphics[width=0.7\linewidth]{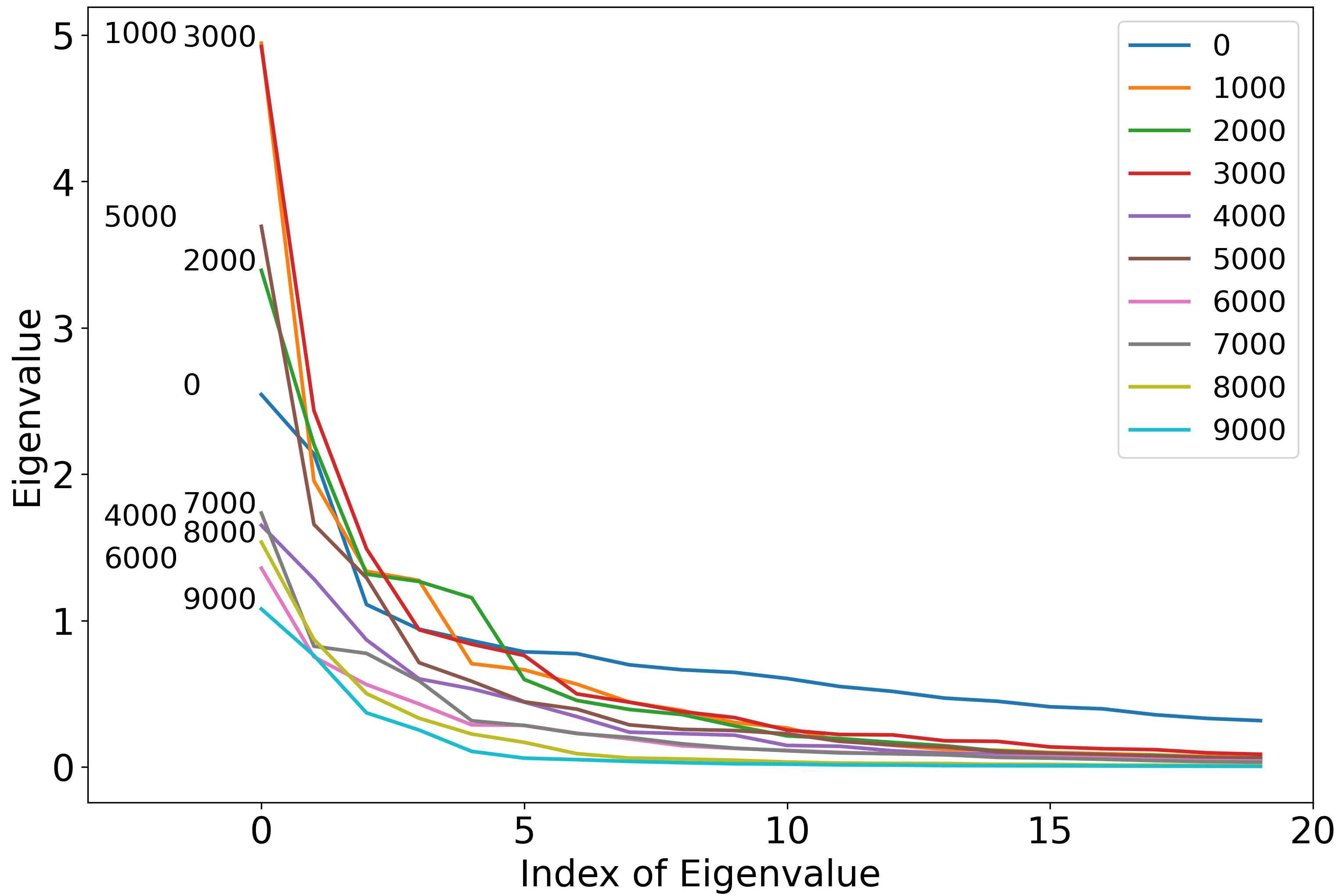}
    \caption{Each curve represents the first 20 eigenvalues obtained using multidimensional scaling (MDS) of the $64\times 64$ normalized distance matrix from $64$ real images under learned metric during training on {\it Shells} dataset. We draw curves for iterations $0,1000,\cdots,9000$.  }
    \label{fig:eigen_20}
    \vspace{-1em}
\end{figure}

Now we observe the eigenvalues to see how the Euclidean distance among feature vectors evolve during the training process. As shown in Figure~\ref{fig:eigen_20}, the curve of eigenvalues becomes closer and closer to $x$-axis with training going forward. At the iteration $9000$, only the first few eigenvalues are non-trivial which implies that the valid feature vectors are compressed on a low-dimensional hyperplane. Compared to the input data dimension $m=1024 \times 1024 \times 3$, 
%and feature space dimension $n=50$, 
the valid feature subspace is significantly flat and uniform. This experimental result is consistent with results in \cite{Dai_and_Hang_2021_ICCV}, even though the training settings being used are very different. 
An example of how to infer the shape of data set using eigenvalue curve is shown in Figure~\ref{fig:flat_example}.

%Blue curve represents the eigenvalue curve of blue point set and red curve represents the eigenvalue curve of red point set. 

The above observation suggests interesting facts to the question at the beginning of this section: %interpolation of valid features is most likely a valid feature.
%From the above observation, we conclude that: 
If a set of valid feature vectors $y_1,\cdots,y_k$ are close to each other, then any interpolation point $y=\sum_{i=1}^k\pi_i y_i$ with $\sum\pi_i=1$ and $0\leq\pi_i\leq 1$ is \textit{very likely} a valid feature vector. 
%In the next section, we show how this observation helps us to approximate the gradient raised by some imaginary real sample. 
\begin{figure}[!ht]
    \centering
    \begin{tabular}{c c}
       \includegraphics[width=0.48\linewidth]{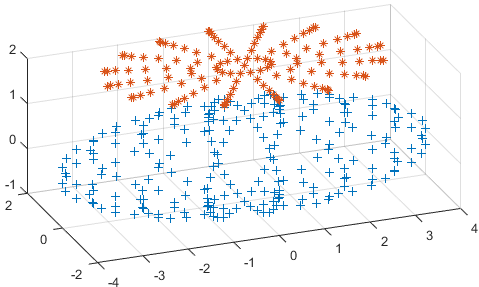}  & \includegraphics[width=0.32\linewidth]{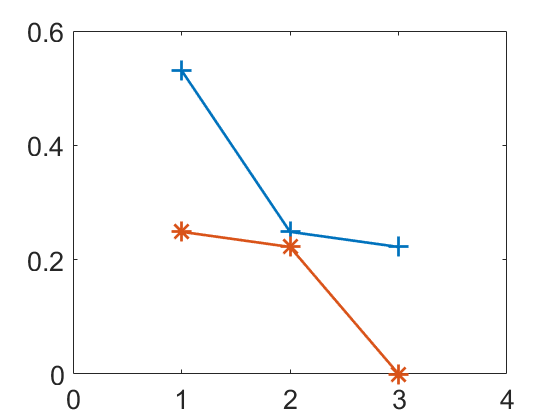}  \\
        (a) & (b)
    \end{tabular}
    \caption{(a) Blue points are located on an ellipsoid and red points are located on a flat disc; (b) Each curve represents the eigenvalue curve of corresponding point set with the same color. }
    \label{fig:flat_example}
    \vspace{-2em}
\end{figure}
Next, we explore how this flattening effect helps with data augmentation.

\subsection{Implicit Data Augmentation}
Given some neural network $g_w$, loss function $L$ and training samples $x_i$, one direct way of augmenting data is to generate synthetic data sample $x$ and use it as training data. All the efforts a synthetic data $x$ could make end up with the calculation of the gradients:
$$\frac{\partial L(g_w(x))}{\partial w}\ .$$

%The traditional data augmentation methods mainly focus on how to derive synthetic data directly from transformations of training data\cite{zhao2020diffaugment, shorten2019survey}. 
Given valid feature vectors $y_1,y_2,\cdots,y_k$ which are extracted from some training samples $x_1,x_2,\cdots,x_k$. %are close to each other in the sense that $\|y_i-y_j\|_2$, $\forall i\neq j$ is small enough, 
For any interpolation $y=\sum_{i=1}^k \pi_i y_i$ with  $0\leq\pi_i\leq 1$ and $\sum_{i=1}^k\pi_i=1$,
based on the flattening effect, \textit{very likely} there exists a virtual real data point $x$ such that $y=g_w(x)$. Even though it is not obvious to construct $x$ explicitly, we are able to estimate its contribution to gradients \textit{implicitly} by taking the average of the contributions of $x_1,\cdots,x_k$:
$$\frac{\partial L(g_w(x))}{\partial w}\approx \sum_{i=1}^k\pi_i\frac{\partial L(g_w(x_i))}{\partial w}\ ,$$
when $y_1,\cdots,y_k$ are close enough.

The above assertion is summarised in the following:

\begin{lemma}
Given some neural network $g_w:\mathbb{R}^m\rightarrow\mathbb{R}^n$ and some differentiable loss function $L:\mathbb{R}^n\rightarrow\mathbb{R}$. Fix $y=g_w(x)$. Then for a set of nearby points $y_i=g_w(x_i),$ $i=1,\cdots,k$ such that $y=\sum_{i=1}^k \pi_i y_i$, $\sum_{i=1}^k \pi_i=1$, $0\leq\pi_i\leq 1$, we have:
\begin{align*}
\bigg| \frac{\partial L(g_w(x))}{\partial w} - \sum_{i=1}^k\pi_i\frac{\partial L(g_w(x_i))}{\partial w} \bigg| =  O(\max_i\|y-y_i\|)\ .\\
\end{align*}
\label{lem:backpropagation}
\vspace{-3em}
\end{lemma}
\begin{proof}
In the following, we use $\partial_j$ to represent the partial derivative of the $j$-th coordinate.
\begin{align*}
 & \frac{\partial L(g_w(x))}{\partial w} - \sum_{i=1}^k\pi_i\frac{\partial L(g_w(x_i))}{\partial w}\\
= & \sum_{j=1}^n\partial_j L(y)\frac{\partial g^{(j)}_w(x)}{\partial w}-\sum_{i=1}^k\pi_i \sum_{j=1}^n\partial_j L(y_i)\frac{\partial g^{(j)}_w(x_i)}{\partial w}\\
= & \sum_{i,j}\partial_j L(y)\pi_i\frac{\partial g^{(j)}_w(x_i)}{\partial w}-\sum_{i,j}\pi_i \partial_j L(y_i)\frac{\partial g^{(j)}_w(x_i)}{\partial w}\\
= & \sum_{i,j}\pi_i\frac{\partial g^{(j)}_w(x_i)}{\partial w}\left(\partial_j L(y)-\partial_j L(y_i)\right)\\
= & \frac{\partial g_w(x)}{\partial w} O(\max_i\|y-y_i\|)=O(\max_i\|y-y_i\|)\\
%= & L'(y)\cdot\frac{\partial\big( \sum_{i=1}^k \pi_i g_w(x_i)\big)}{\partial w}-\sum_{i=1}^k\pi_i L'(y_i)\frac{\partial g_w(x_i)}{\partial w}\\
%= & \sum_{i=1}^k\pi_i \left(L'(y)-L'(y_i)\right)\cdot\frac{\partial g_w(x_i)}{\partial w}\\
%= & \sum_{i=1}^k\pi_i \cdot\frac{\partial g_w(x_i)}{\partial w}\left[L''(y)(y-y_i)+O(\|y-y_i\|^2)\right].\\
\end{align*}
\end{proof}

In summary, one can update the network parameters by using the average of the gradients raised by some set of training samples when performing gradient descent, if their embedded features are close to each other. In the following sections, we introduce the proposed data-driven augmentation algorithm in detail.

\subsection{Nearest Neighbors Interpolation}
%k near neighbors interpolation from metric embeddings. This requires discriminator to have the effect of metric learning.

Denote a data point (image) $x_i$, its feature vector $y_i = g_w(x_i)$, and the $k$ nearest neighbours (including itself) as $y_{ij}, j = 1,\cdots, k$. We define an interpolated feature for $y_i$ using its $k$ nearest neighbours as 
\begin{equation}\label{eq:interp}
\tilde{y_i} = \sum_{j=1}^{k} \pi_{ij} y_{ij}  
\end{equation}
where $\sum_{j=1}^{k} \pi_{ij} = 1$ and $0 \leq \pi_{ij} \leq 1$. 

For each $y_i$, $\pi_{ij}$ in Eqn~(\ref{eq:interp}) follows Dirichlet distribution: 
\begin{equation}
\pi_{ij} \sim \text{Dir}(\alpha_{ij}),\  i = 1, 2, \cdots, k
\label{eqn:dirichlet}
\end{equation}
One can decide the concentration parameters $\alpha_{ij}$ to control the weights of the nearest neighbors. 
For example, when $\alpha_{ij} = 1$ for all $j$s, the weights are uniformly distributed. 
Here we leverage distances between features and geometry of manifold to inform the parameters. The detailed procedure is described as follows.

For the $i$-th feature $y_i$ and its nearest neighbor $y_{ij},\ j = 1,2,\cdots,k$:
\begin{equation}
\alpha_{ij} = T(M(y_i,y_{ij}))^t \ ,
\label{eqn:alpha}
\end{equation}
where $T(x): \mathbb{R}^* \rightarrow \mathbb{R}^+$ is a monotonically decreasing function. ($\mathbb{R}^*=\{x \in \mathbb{R}, x >= 0\}, \mathbb{R}^+=\{x \in \mathbb{R}, x >0\}$). $M(y_i,y_{ij})$ is the distance between $y_i$ and $y_{ij}$ and $t > 0$ is used to control the skewness of the interpolation. 
There are lots of choices for $T(x)$,  for example, $T(x) = \frac{1}{1+x}$. 
The intuition is to have larger weights for closer neighbors, as shown in Figure~\ref{fig:interp_dirichlet}.
In terms of $t$, a smaller $t$ gives more uniform/smooth interpolation while larger $t$ prefers more weights on nearer neighbors. 
%We will discuss more about $t$ by incorporating the manifold information in the next section.
For simplicity we set $t=1$ as the default choice.
%More relevant discussions can be found in supplementary material.
\begin{figure}[!ht]
    \centering
    \includegraphics[width=0.8\linewidth]{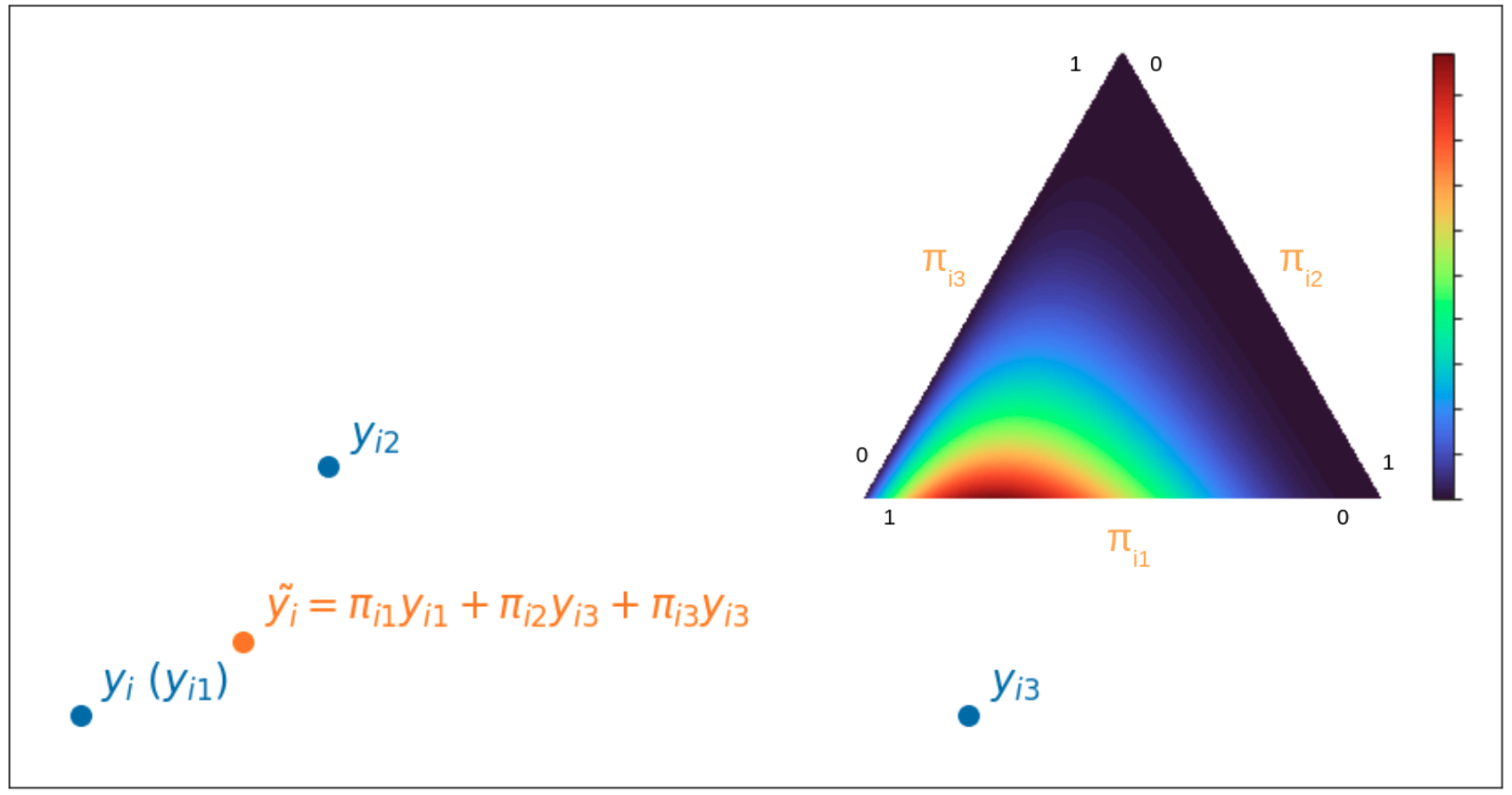}
    \caption{Illustration of using Dirichlet distribution to interpolate features. $y_{ij}, j = 1,2,3$ are the first 3 nearest neighbors of $y_{i}$. $\tilde{y_i} = \pi_{i1}y_{i1}+\pi_{i3}y_{i2}+\pi_{i3}y_{i3}$ is the interpolation where $\pi_{ij}$ are sampled from Dirichlet distribution with concentration parameter $\alpha_{ij}$ in proportion to the distance $M(y_{i},y_{ij})$. }
    \label{fig:interp_dirichlet}
    \vspace{-2em}
\end{figure}

\subsection{Data-Driven Adaptive Augmentation} 
To facilitate usage of feature interpolation, 
we consider taking advantage of the flattening effect during training to decide the aggressiveness of augmentation. 
Under the context, the aggressiveness can mainly be interpreted by (1) choices of nearest neighbour interpolation, (2) shape of Dirichlet distribution and (3) proportion of augmented features to use.
To address (1), when the embedding space is more flat, one may use a larger $k$ for sampling interpolated points. On the contrary, a small $k$ leads to the interpolated features only cover a small portion of the feature space, thus may result in limited augmentation effect and bias in recovering real data manifold. 
One important question is how to reflect the degree of ``flatness'' of data manifold. 
As mentioned earlier, such information can be reached by the (multidimensional scaling) MDS of pairwise distances between features \cite{Shao_2018_CVPR_Workshops,Dai_and_Hang_2021_ICCV}. 
%More amount of small eigenvalues of MDS indicate the shape of manifold is more ``flat''. 
The flatness can be reflected by the number of large eigenvalues of MDS, where fewer number of large eigenvalues indicate approximately smaller dimensions of the space. 
Empirically we count number of eigenvalues $\{\lambda_i\}$ bigger than $10\%$ of the largest eigenvalue $\lambda_{max}$ in a batch $b$ as the effective dimensions, and use $k =  I(\lambda_i < 0.1 \lambda_{max})$ as number of near neighbours. 

The shape of Dirichlet distribution is controlled by $M$ obtained from data itself as discussed in Section 3.3.
%Although one can further define data-driven criteria for deciding $t$, for simplicity we fix $t=1$, which already leads to satisfying results in experiments.
We also involve augmentation probability $p$ which decides the proportion of interpolated points used for training. $p=0$ refers to no augmentation, and $p=1$ refers to when all features are from interpolation. Similarly, we let $p=(k-1) / b$ which introduces more aggressive augmentation with fewer effective dimensions. In practice one can also find other ways to define $p$, or use fixed $p$s for simplicity. In experiments we observe using a reasonable choice of $p$ (such as $p=0.6$) is sufficient for stabilizing training. We will discuss the behaviors of these parameters later in Section 4.2.
The whole Adaptive Feature Interpolation (AFI) algorithm is summarized in Algorithm 1.%~\ref{algo:1}.
\begin{algorithm}~\label{algo:1}
	\caption{Adaptive Feature Interpolation.}
	%\hspace*{\algorithmicindent} %\textbf{Require:}  $m$,  the  batch  size.\\ %$ncritic$, the number of iterations of the critic per generator iteration.\\
	\hspace*{\algorithmicindent} \textbf{Input:}  A batch of features $\{y_i\}$ extracted from real data;\\
	\hspace*{\algorithmicindent} \textbf{Output:} Augmented batch of features $\{y_i^*\}$;\\
	\begin{algorithmic}[1]
        \STATE Calculate distance matrix $M$ from $\{y_i\}$;
        \STATE Solve for MDS of $M$ and return its eigenvalues $\{\lambda_i\}$;
        \STATE Calculate $\{\alpha\}, k, p$ from $M$ and $ \{\lambda_i\}$; %using Eqn~(\ref{eqn:alpha},\ref{eqn:dyanmic_P},\ref{eqn:t});
        \STATE For each $i$, sample interpolated features $\tilde{y_i}$ using Eqn~(\ref{eqn:dirichlet}) with its $k$ near neighbours;
		\STATE For each $i$, set $y_i^*=\tilde{y_i}$ using Eqn~(\ref{eq:interp}) with probability $p$ else $y_i^*=y_i$.
	\end{algorithmic}
\end{algorithm}

\section{Experiments}
In this section, we explore the behavior of the proposed method and provide evaluation results on multiple datasets.

\subsection{Datasets and Implementation Details}
We conducted unconditional generation experiments on several benchmark datasets or their subsets, including Shells, Art, Anime Face, Pokemon provided by \cite{liu2021}, Cat, Dog \cite{Si2012LearningHI}, Obama, Grumpy cat \cite{zhao2020diffaugment} 
%LSUN church, LSUN bridge \cite{lsun},  FFHQ \cite{Karras_2019_CVPR}, 
 and CIFAR-10 \cite{cifar}. 
We utilized various metrics for evaluations, including Frechet Inception Distance (FID) \cite{NIPS2017_8a1d6947}, Kernel Inception Distance (KID) \cite{kid}
%, Inception Score (IS) \cite{salimans2016improved} 
and Precision and Recall (PR) \cite{precision_recall_distributions}. By default we generated 50K samples against real data for evaluation.
%We also include DScore as an additional measurement of diversity of generated images as introduced in Section 3.
%For both metrics we use 1000 synthesized data against real data set for evaluation. 
Lower FID, KID scores and higher PR indicate better results. 

%\subsection{Implementation Details}
We adopted StyleGAN2 \cite{Karras2020ada}  and FastGAN \cite{liu2021} network architectures, and used consistent parameter settings provided in their papers for experimentation. 
To facilitate experiments with multidimensional discriminator output, the number of output neurons $n$ of discriminator in \cite{Karras2020ada} was set to $20$, and the output logits of discriminator in \cite{liu2021} were reshaped to fit multidimensional setting.
Feature interpolation is performed on features extracted from both real and fake images, which is empirically shown to be beneficial in experiments.
Experiments were conducted using PyTorch framework on Tesla V100 GPUs. 

\subsection{Ablation Study}
%Our baseline is the same as the default setting in \cite{liu2021}.
In this section we study the behaviors of proposed feature interpolation (FI) along with experiments using direct image interpolation (II) for comparison. 

We first study the effect of $p$ in simple cases using StyleGAN2 architectures \cite{Karras2020ada,Karras2019stylegan2}.
In each setting we recorded FID during training with $p=0.3,0.6,0.9$ as shown in Figure~\ref{fig:comp_vs_iter}. Here in II sessions we performed direct interpolation on images in a {\it mixup} style \cite{zhang2018mixup}. For FI sessions we did not utilize any image augmentation techniques.
One can see that compared with FI experiments, the II sessions diverged earlier in all cases. The II session with $p=0.9$ has less positive effect compared to when $p=0.6$, which indicates image interpolation does not favor large $p$s in this case.
In contrast, the FI experiments had more stable training sessions even with $p=0.9$. 
Best results were obtained with dynamic $p$ in this case.
These results suggest that FI may better enjoy the ``non-leaking'' property \cite{Karras2020ada} in training. 
In experiments on different datasets we find the use of small amount of original valid features is a necessary regularization for stable training.
%We have also tried using $p=1$, but did not find satisfying and stable sessions. Indeed, 
Note that for II it is not obvious to apply dynamic $k$s. One reason is that using Euclidean distances between pixels to find near neighbors seems worthless. In addition, with fixed real images one cannot implement dynamic augmentation based on the flattening effect.
\begin{figure}[!ht]
    \centering
    \includegraphics[width=0.8\linewidth]{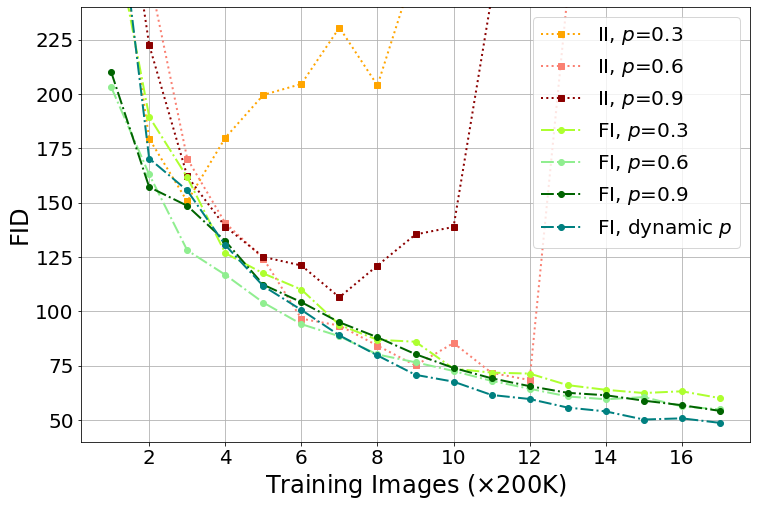}
    \caption{FID evaluation during training on {\it Shells} dataset  using image interpolation (II) and feature interpolation (FI) with StyleGAN2 architectures.}
    \label{fig:comp_vs_iter}
\end{figure}
    
Next we experimented on settings both uniform ($t=0$) and skewed ($t=1$) distributions to study the behavior of Dirichlet distributions on feature interpolation.
In each experiment we used batch size $8$, and trained 50K iterations with fixed $p=0.9$ using FastGAN architectures.
In this case we employed DiffAugment for image augmentation and used 1K generated images for fast evaluation.
In each setting we report the best FID during training along with its corresponding iteration in Table~\ref{tab:shells_p09}. 
\begin{table}[!ht]
  \small
  \centering
  \begin{tabular}{l|l|cccccccc}
    \toprule
    & $k$ & 1 & 2 & 3 & 4 & 5 & 6 & 7 & 8\\
    \midrule
    \multirow{2}*{\makecell{Uniform\\(t=0)}}
    %\multirow{4}*{111}%{ \makecell{Uniform \\ ($t=0$)}} 
    &FID &165.72 &140.89 &148.95 &141.82 &145.58 &138.54 &136.13 &144.90  \\
    &Iter(K) & 10 &30 &20 &40 &35 &45 &35 &20 \\
    \midrule
    \multirow{2}*{\makecell{Skewed\\(t=1)}}
    &FID &165.72 &130.98 &136.33 &141.51 &131.54 &137.62 &139.90 & 135.85\\
    &Iter(K) &10  &35 &50 &40 &50 &35 &35 &45 \\
    \bottomrule
  \end{tabular}
  \caption{FID evaluation of generated samples on {\it Shells} dataset using FastGAN architectures with fixed $k$s and $p=0.9$.}
  \label{tab:shells_p09}
  \vspace{-2em}
\end{table}
 Overall, we notice that FIDs with skewed distributions are in general better than ones with uniform distributions. 
 Intuitively, with skewed distribution the interpolated features are likely closer to original features, thus may lead to smaller bias in training. 
 
One interesting question is the relation between size of dataset $N$ and value of $k$ in training. We recorded $k$s up to 500K training images in CIFAR-10 experiments and 200K training images in {\it Art} experiments with effective batch size $8$ on each GPU, and computed averaged $k$ against different $N$s. 
 \begin{table}[!ht]
	\centering
	%\small
	\begin{tabular}{lcccc|cc}
	    \hline
	    %\toprule
         &  \multicolumn{4}{c}{CIFAR-10} & \multicolumn{2}{c}{Art}\\
		\hline
		Images & 100 & 500 & 5000 & 50000 & 100 & 1000  \\
		\hline
		$p=0$ &4.39 &4.02 &2.34 &2.14 &3.41 &2.43 \\
		%\hline
		AFI  &4.62 &3.86 &2.18 &2.26 &3.53 &2.59\\
		\hline
	\end{tabular}
	\caption{Averaged $k$ during training under different amount of training data $N$ with batch size 8 on each GPU.}
	\label{tab:k_vs_N}
  \vspace{-2em}
\end{table}
Table~\ref{tab:k_vs_N} also shows that under each setting, overall averaged $k$ decreases with $N$ increasing, which corresponds to less augmentation with more real training data.
This dynamic mechanism reduces the risk of introducing more biased augmentation with larger $N$ as mentioned in \cite{Karras2020ada}, where 
authors point out an interesting phenomenon that the positive effect of data augmentation decreases with size of real training data increasing.
In addition, results in Table~\ref{tab:k_vs_N} also reveal that more training data enlarges the effective dimension of data manifold. 
Here we also present averaged $k$ during training on CIFAR-10 in Figure~\ref{fig:k_vs_N} which provides some insights on the behaviors of flatness reflected by $k$. We notice that  smaller $N$s suggest fewer effective dimensions of data manifold and stronger augmentation. Especially with small $N$s ($N=100$ and $500$), the use of augmentation becomes more and more aggressive during training, except for the beginning of training sessions. At the beginning stage, the averaged $k$ fluctuates or even decreases before it reaches the point of inflection. The longer period of this training phase for larger $N$s may suggest that the discriminator needs more training to learn meaningful feature representations.
\begin{figure}[!ht]
    \centering
    \includegraphics[width=0.8\linewidth]{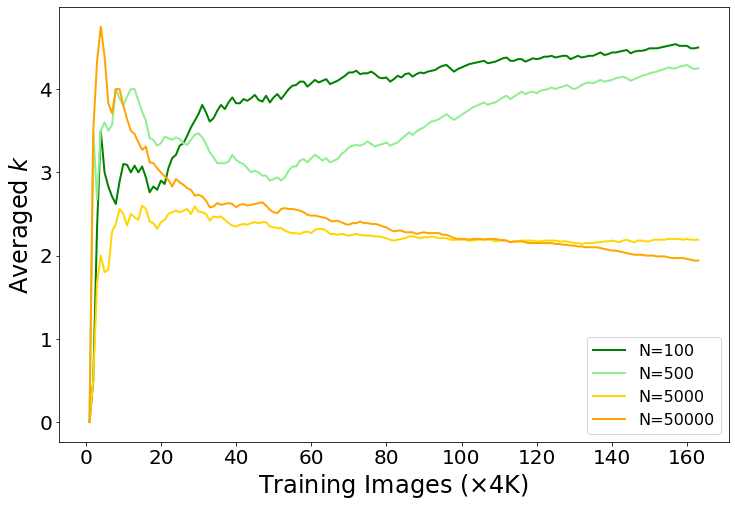}
    \caption{Averaged $k$ during training on CIFAR-10 with different  $N$s using effective batch size 8 on each GPU.}
    \label{fig:k_vs_N}
\end{figure}

To study the effect of feature interpolation on other datasets, we performed experiments using AFI without applying transformation-based image augmentations such as \cite{Karras2020ada,zhao2020diffaugment}. We also present results using interpolation methods similar as input and manifold mixup \cite{zhang2018mixup,verma2019manifold}, except that in the unsupervised task no labels are available for interpolation.
In each setting we augmented both real and fake images (or features) with the corresponding method only. 
Interpolations using \cite{zhang2018mixup,verma2019manifold} were implemented on images and features respectively.
Table~\ref{tab:comp_interp_only} shows AFI significantly improves baseline results  on all datasets. In contrast, directly applying feature interpolation in a {\it mixup} style led to the worst result. 
\begin{table}[!ht]
   %\scriptsize  
   %\small
	\centering
	%\caption{Main differences between GANs and our framework.}
	\begin{tabular}{lcccccccc}
	    \hline
	       & Shells & Anime & Art & Pokemon  & Cat & Dog & Grumpy cat & Obama \\
		\hline
	    StyleGAN2  &229.83   &197.01 &145.08 &194.93  &78.32   & 225.63 &43.88 &104.64 
	    \\
	    \hline
        + Input &135.66 &91.99 &86.51 &174.91 &60.35 &118.79 &26.98 &36.65\\	  
        + Feature  &319.80 &308.16 &117.35 &376.83 &236.99 &270.66 &118.80 &160.57 \\
	    + AFI & 48.71   & 85.97 & 81.08  & 131.86  & 57.47  & 143.60 & 37.46 & 31.88
	    \\
		\hline
	\end{tabular}
	\caption{FID evaluation of experiments without transformation-based augmentations.}
	\label{tab:comp_interp_only}
    %\vspace{-2em}
\end{table}

\subsection{Results}
In the following we provide final evaluations results on various datasets. 
\begin{table}[!ht]
   %\scriptsize  
   %\small
	\centering
	%\caption{Main differences between GANs and our framework.}
	\begin{tabular}{lcccccccc}
	    \hline
	    Dataset & Shells &  Anime & Art  & Pokemon & Dog & Cat & Grumpy cat & Obama  \\
	    \hline
	    Image size & 1024 & 1024 & 1024 & 1024 & 256 & 256 & 256 & 256 \\
	    \hline
	    Number of images & 64 & 120 & 1000  & 800 &389 &160 &100 &100 \\
	    \hline
	    FastGAN \cite{liu2021} &152.53  &60.04  & 48.44 &57.05  &51.24 &39.30 &27.59 &40.52 \\
	    + AFI &{\bf 124.80}  &{\bf 55.35}  &{\bf 43.09}  &{\bf 50.47}  & {\bf 50.89} &{\bf 35.18} &{\bf 25.02} &{\bf 36.43} \\
		\hline
		%\multirow{3}*{StyleGAN2}
	    StyleGAN2 \cite{Karras2020ada}  &123.66 &60.51  &72.36   &75.39 &59.07 &39.78 &31.58 &44.04\\
	    Hinge loss, $n=20$ & 101.72 &55.67 & 58.42 & 55.32 &56.43 &40.24 &28.69 &40.18\\
	    %\hline
	    + AFI &{\bf 62.99} &{\bf 33.48}  &{\bf 43.94}  &{\bf 44.79} &{\bf 49.14} &{\bf 35.26} &{\bf 22.03} &{\bf 31.99}
	    \\
		\hline
	\end{tabular}
	\caption{FID evaluations on $1024\times1024$ and $256\times256$ experiments using FastGAN and StyleGAN2  architectures.}
	\label{tab:comp_1024}
    \vspace{-2em}
\end{table}

FID evaluation of $1024\times1024$ and $256\times256$ experiments are displayed in Table~\ref{tab:comp_1024}. 
%Table~\ref{tab:sets100} shows recorded best FID along with its DScore. 
In these experiments we incorporated image augmentation techniques, where \cite{Karras2020ada} applies adaptive augmentation and \cite{liu2021} employs DiffAugment.
%For fair comparison we used consistent training settings across experiments. 
Table~\ref{tab:comp_1024} shows that using feature interpolation further improved results from strong baselines. With StyleGAN2 architectures we observe more significant gains across datasets.
Note that simply using $n=20$ without FI, one can already see improvements compared to results from \cite{Karras2020ada}. This behavior is consistent with theoretical analysis and experimental results in \cite{https://doi.org/10.48550/arxiv.2109.03378} that multidimensional discriminator output has its advantage in GAN training.
We further present evaluations of KID and precision-recall with StyleGAN2 architectures in Table~\ref{tab:kid_1024} and Table~\ref{tab:kid_256} for reference.
Examples of randomly generated $1024\times1024$ images are displayed in Figure~\ref{fig:comp_gen}.
As shown in the figure, samples from experiments with FI have consistently better qualities.

We also tested the effect of feature interpolation on CIFAR-10 with small amount of partial data using the default setting in \cite{Karras2020ada} as baseline. Using feature interpolation improves FID from 42.80 to 27.62 with only 0.2\% training data (100 images), and from 19.69 to 13.50 with 1\% training data. 
\begin{table}[!ht]
	\centering
	\small
	\begin{tabular}{lcc|cc|cc|cc}
	    \hline
        & \multicolumn{2}{c}{Shells} & \multicolumn{2}{c}{Anime} & \multicolumn{2}{c}{Art} & \multicolumn{2}{c}{Pokemon}\\
		\hline
		 & KID  & PR & KID  & PR & KID  & PR & KID  & PR \\
		\hline
		\cite{Karras2020ada}   &20  &(0.789,0.085) &15  &(0.966,0.933) &26  &(0.574,0.823) &28  &(0.621,0.727)  \\
		+ AFI  &{\bf 2}  &({\bf 0.852},{\bf 0.132}) &{\bf 4} &({\bf 0.984},{\bf 0.974}) &{\bf 9}  &({\bf 0.887},{\bf 0.965}) &{\bf 12} &({\bf 0.948},{\bf 0.922}) \\
		\hline
	\end{tabular}
	\caption{KID(x${10}^3$)($\downarrow$) and Precision-Recall (PR)($\uparrow$) evaluations on $1024\times1024$ experiments.}
	\label{tab:kid_1024}
    \vspace{-2em}
\end{table}

\begin{table}[!ht]
	\centering
	\small
	\begin{tabular}{lcc|cc|cc|cc}
	    \hline
        & \multicolumn{2}{c}{Dog} & \multicolumn{2}{c}{Cat} & \multicolumn{2}{c}{Grumpy cat} & \multicolumn{2}{c}{Obama}\\
		\hline
		 & KID  & PR & KID  & PR & KID  & PR & KID  & PR \\
		\hline
		\cite{Karras2020ada}   &18  &(0.874,{\bf 0.948}) &6  &({\bf 0.974},{\bf 0.951}) &5  &(0.845,0.794)  &13  &(0.930,0.860)  \\
		+ AFI  &{\bf 14}  &({\bf0.922},0.932) &{\bf 4}  &({\bf 0.976},{\bf 0.950}) &{\bf 3}  &({\bf 0.973},{\bf 0.953}) &{\bf 10} &({\bf 0.979},{\bf 0.970}) \\
		\hline
	\end{tabular}
	\caption{KID(x${10}^3$) and Precision-Recall (PR) evaluations on $256\times256$ experiments.}
	\label{tab:kid_256}
\end{table}
\begin{figure}[!ht]
    \centering
    \includegraphics[width=1\linewidth]{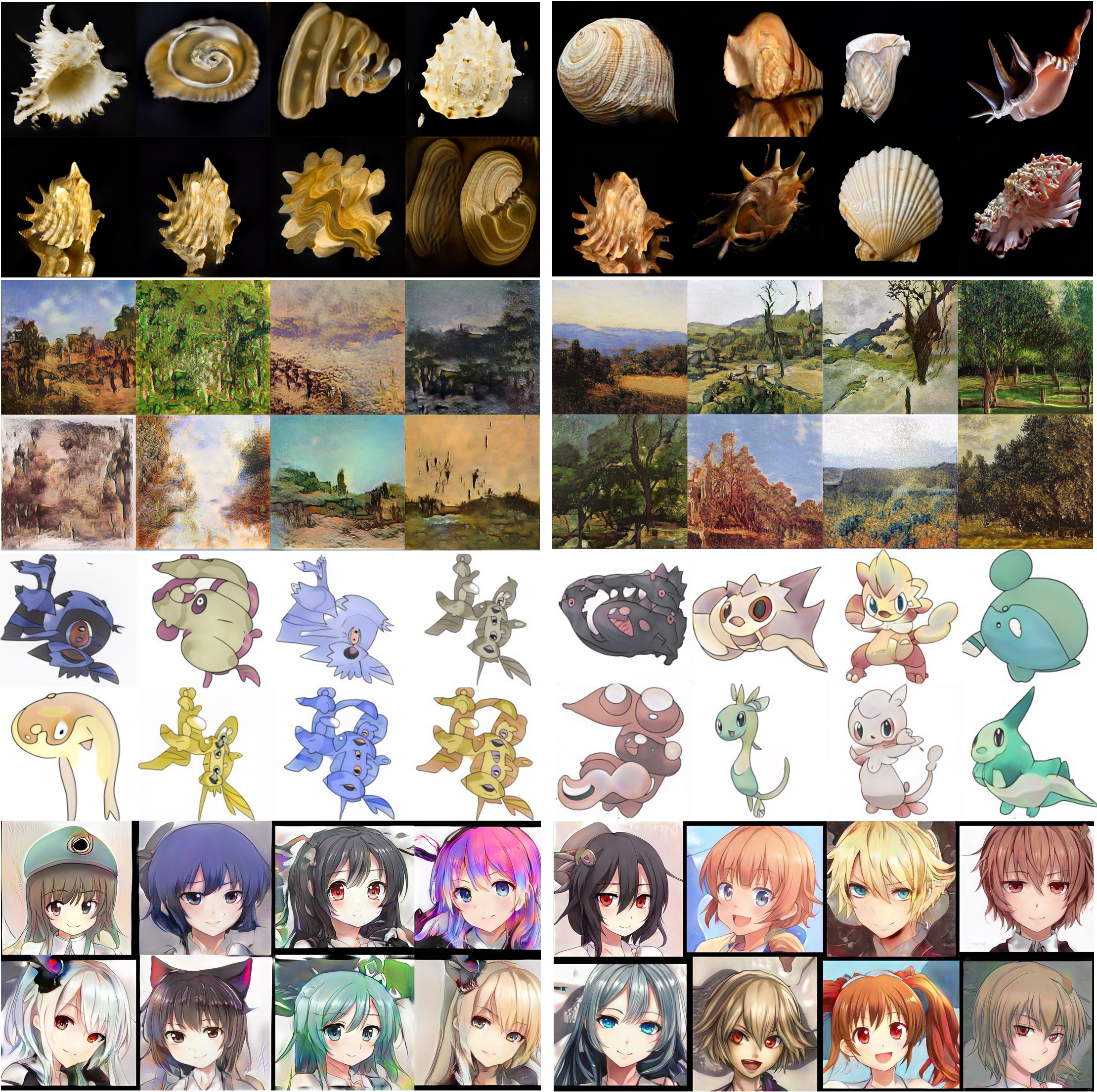}
    \caption{Randomly generated $1024\times1024$ samples on different datasets. From top to bottom: Shells, Art, Pokemon and Anime.  Left: StyleGAN2-ADA \cite{Karras2020ada}; Right: with Adaptive Feature Interpolation.}
    \label{fig:comp_gen}
\end{figure}

\section{Discussion}
In this paper we have proposed an adaptive augmentation approach for low-shot data generation. Instead of producing new training images, the method functions in the multidimensional feature space of discriminator output by utilizing the flattening effect of feature space during training.  
Experiments show the proposed method improves results from strong baselines in low-data regime.

% \clearpage\mbox{}Page \thepage\ of the manuscript.
% \clearpage\mbox{}Page \thepage\ of the manuscript.

% This is the last page of the manuscript.
% Now we have reached the maximum size of the ECCV 2022 submission (excluding references).
% References should start immediately after the main text, but can continue on p.15 if needed.

\clearpage
\bibliographystyle{splncs04}
\bibliography{egbib}

\begin{thebibliography}{10}
\providecommand{\url}[1]{\texttt{#1}}
\providecommand{\urlprefix}{URL }
\providecommand{\doi}[1]{https://doi.org/#1}

\bibitem{arandjelovic2012three}
Arandjelovi{\'c}, R., Zisserman, A.: Three things everyone should know to
  improve object retrieval. In: 2012 IEEE Conference on Computer Vision and
  Pattern Recognition. pp. 2911--2918. IEEE (2012)

\bibitem{wgan}
Arjovsky, M., Chintala, S., Bottou, L.: {W}asserstein generative adversarial
  networks. In: Proceedings of the 34th International Conference on Machine
  Learning. Proceedings of Machine Learning Research, vol.~70, pp. 214--223.
  PMLR (06--11 Aug 2017)

\bibitem{r2021learning}
Balestriero, R., Pesenti, J., LeCun, Y.: Learning in high dimension always
  amounts to extrapolation (2021)

\bibitem{kid}
Bińkowski, M., Sutherland, D.J., Arbel, M., Gretton, A.: Demystifying {MMD}
  {GAN}s. In: International Conference on Learning Representations (2018)

\bibitem{Choe_2017_ICCV}
Choe, J., Park, S., Kim, K., Hyun~Park, J., Kim, D., Shim, H.: Face generation
  for low-shot learning using generative adversarial networks. In: Proceedings
  of the IEEE International Conference on Computer Vision (ICCV) Workshops (Oct
  2017)

\bibitem{chum2011total}
Chum, O., Mikulik, A., Perdoch, M., Matas, J.: Total recall ii: Query expansion
  revisited. In: CVPR 2011. pp. 889--896. IEEE (2011)

\bibitem{chum2007total}
Chum, O., Philbin, J., Sivic, J., Isard, M., Zisserman, A.: Total recall:
  Automatic query expansion with a generative feature model for object
  retrieval. In: 2007 IEEE 11th International Conference on Computer Vision.
  pp.~1--8. IEEE (2007)

\bibitem{Dai_and_Hang_2021_ICCV}
Dai, M., Hang, H.: Manifold matching via deep metric learning for generative
  modeling. In: Proceedings of the IEEE/CVF International Conference on
  Computer Vision (ICCV). pp. 6587--6597 (October 2021)

\bibitem{https://doi.org/10.48550/arxiv.2109.03378}
Dai, M., Hang, H., Srivastava, A.: Rethinking multidimensional discriminator
  output for generative adversarial networks (2021)

\bibitem{devries2017dataset}
DeVries, T., Taylor, G.W.: Dataset augmentation in feature space. arXiv
  preprint arXiv:1702.05538  (2017)

\bibitem{gan}
Goodfellow, I., Pouget-Abadie, J., Mirza, M., Xu, B., Warde-Farley, D., Ozair,
  S., Courville, A., Bengio, Y.: Generative adversarial nets. In: Advances in
  Neural Information Processing Systems. vol.~27 (2014)

\bibitem{NIPS2017_8a1d6947}
Heusel, M., Ramsauer, H., Unterthiner, T., Nessler, B., Hochreiter, S.: Gans
  trained by a two time-scale update rule converge to a local nash equilibrium.
  In: Advances in Neural Information Processing Systems. vol.~30 (2017)

\bibitem{Hong_2021_CVPR}
Hong, J., Fang, P., Li, W., Zhang, T., Simon, C., Harandi, M., Petersson, L.:
  Reinforced attention for few-shot learning and beyond. In: Proceedings of the
  IEEE/CVF Conference on Computer Vision and Pattern Recognition (CVPR). pp.
  913--923 (June 2021)

\bibitem{hong2020deltagan}
Hong, Y., Niu, L., Zhang, J., Liang, J., Zhang, L.: Deltagan: Towards diverse
  few-shot image generation with sample-specific delta. arXiv preprint
  arXiv:2009.08753  (2020)

\bibitem{hong2020f2gan}
Hong, Y., Niu, L., Zhang, J., Zhao, W., Fu, C., Zhang, L.: F2gan:
  Fusing-and-filling gan for few-shot image generation. In: Proceedings of the
  28th ACM International Conference on Multimedia. pp. 2535--2543 (2020)

\bibitem{Huang_2021_ICCV}
Huang, K., Geng, J., Jiang, W., Deng, X., Xu, Z.: Pseudo-loss confidence metric
  for semi-supervised few-shot learning. In: Proceedings of the IEEE/CVF
  International Conference on Computer Vision (ICCV). pp. 8671--8680 (October
  2021)

\bibitem{Karras2020ada}
Karras, T., Aittala, M., Hellsten, J., Laine, S., Lehtinen, J., Aila, T.:
  Training generative adversarial networks with limited data. In: Proc. NeurIPS
  (2020)

\bibitem{Karras2019stylegan2}
Karras, T., Laine, S., Aittala, M., Hellsten, J., Lehtinen, J., Aila, T.:
  Analyzing and improving the image quality of {StyleGAN}. In: Proc. CVPR
  (2020)

\bibitem{kimICML20}
Kim, J.H., Choo, W., Song, H.O.: Puzzle mix: Exploiting saliency and local
  statistics for optimal mixup. In: International Conference on Machine
  Learning (ICML) (2020)

\bibitem{kim2021comixup}
Kim, J., Choo, W., Jeong, H., Song, H.O.: Co-mixup: Saliency guided joint mixup
  with supermodular diversity. In: International Conference on Learning
  Representations (2021)

\bibitem{ko2020embedding}
Ko, B., Gu, G.: Embedding expansion: Augmentation in embedding space for deep
  metric learning. In: Proceedings of the IEEE Conference on Computer Vision
  and Pattern Recognition (2020)

\bibitem{cifar}
Krizhevsky, A., Nair, V., Hinton, G.: Cifar-10 (canadian institute for advanced
  research)

\bibitem{li2020few}
Li, Y., Zhang, R., Lu, J., Shechtman, E.: Few-shot image generation with
  elastic weight consolidation. arXiv preprint arXiv:2012.02780  (2020)

\bibitem{Li_2021_CVPR}
Li, Y., Zhu, H., Cheng, Y., Wang, W., Teo, C.S., Xiang, C., Vadakkepat, P.,
  Lee, T.H.: Few-shot object detection via classification refinement and
  distractor retreatment. In: Proceedings of the IEEE/CVF Conference on
  Computer Vision and Pattern Recognition (CVPR). pp. 15395--15403 (June 2021)

\bibitem{lim2017geometric}
Lim, J.H., Ye, J.C.: Geometric gan (2017)

\bibitem{liu2021}
Liu, B., Zhu, Y., Song, K., Elgammal, A.: Towards faster and stabilized
  {\{}gan{\}} training for high-fidelity few-shot image synthesis. In:
  International Conference on Learning Representations (2021)

\bibitem{Liu_2021_ICCV}
Liu, S., Wang, Y.: Few-shot learning with online self-distillation. In:
  Proceedings of the IEEE/CVF International Conference on Computer Vision
  (ICCV) Workshops. pp. 1067--1070 (October 2021)

\bibitem{ojha2021few-shot-gan}
Ojha, U., Li, Y., Lu, C., Efros, A.A., Lee, Y.J., Shechtman, E., Zhang, R.:
  Few-shot image generation via cross-domain correspondence. In: CVPR (2021)

\bibitem{precision_recall_distributions}
Sajjadi, M.S.M., Bachem, O., Lu{\v c}i{\'c}, M., Bousquet, O., Gelly, S.:
  {Assessing Generative Models via Precision and Recall}. In: {Advances in
  Neural Information Processing Systems (NeurIPS)} (2018)

\bibitem{Shao_2018_CVPR_Workshops}
Shao, H., Kumar, A., Thomas~Fletcher, P.: The riemannian geometry of deep
  generative models. In: Proceedings of the IEEE Conference on Computer Vision
  and Pattern Recognition (CVPR) Workshops (June 2018)

\bibitem{Si2012LearningHI}
Si, Z., Zhu, S.C.: Learning hybrid image templates (hit) by information
  projection. IEEE Transactions on Pattern Analysis and Machine Intelligence
  \textbf{34},  1354--1367 (2012)

\bibitem{Stojanov_2021_CVPR}
Stojanov, S., Thai, A., Rehg, J.M.: Using shape to categorize: Low-shot
  learning with an explicit shape bias. In: Proceedings of the IEEE/CVF
  Conference on Computer Vision and Pattern Recognition (CVPR). pp. 1798--1808
  (June 2021)

\bibitem{NEURIPS2019_36ad8b5f}
Thulasidasan, S., Chennupati, G., Bilmes, J.A., Bhattacharya, T., Michalak, S.:
  On mixup training: Improved calibration and predictive uncertainty for deep
  neural networks. In: Advances in Neural Information Processing Systems.
  vol.~32. Curran Associates, Inc. (2019)

\bibitem{Tian2020RethinkingFI}
Tian, Y., Wang, Y., Krishnan, D., Tenenbaum, J.B., Isola, P.: Rethinking
  few-shot image classification: a good embedding is all you need? ArXiv
  \textbf{abs/2003.11539} (2020)

\bibitem{lecamgan}
Tseng, H.Y., Jiang, L., Liu, C., Yang, M.H., Yang, W.: Regularing generative
  adversarial networks under limited data. In: CVPR (2021)

\bibitem{turcot2009better}
Turcot, P., Lowe, D.G.: Better matching with fewer features: The selection of
  useful features in large database recognition problems. In: 2009 IEEE 12th
  International Conference on Computer Vision Workshops, ICCV Workshops. pp.
  2109--2116. IEEE (2009)

\bibitem{verma2019manifold}
Verma, V., Lamb, A., Beckham, C., Najafi, A., Mitliagkas, I., Lopez-Paz, D.,
  Bengio, Y.: Manifold mixup: Better representations by interpolating hidden
  states. In: International Conference on Machine Learning. pp. 6438--6447.
  PMLR (2019)

\bibitem{Yao2020Automated}
Yao, H., Wu, X., Tao, Z., Li, Y., Ding, B., Li, R., Li, Z.: Automated
  relational meta-learning. In: International Conference on Learning
  Representations (2020)

\bibitem{Yun_2019_ICCV}
Yun, S., Han, D., Oh, S.J., Chun, S., Choe, J., Yoo, Y.: Cutmix: Regularization
  strategy to train strong classifiers with localizable features. In:
  Proceedings of the IEEE/CVF International Conference on Computer Vision
  (ICCV) (October 2019)

\bibitem{Zhang_2021_CVPR}
Zhang, C., Song, N., Lin, G., Zheng, Y., Pan, P., Xu, Y.: Few-shot incremental
  learning with continually evolved classifiers. In: Proceedings of the
  IEEE/CVF Conference on Computer Vision and Pattern Recognition (CVPR). pp.
  12455--12464 (June 2021)

\bibitem{Zhang2020Consistency}
Zhang, H., Zhang, Z., Odena, A., Lee, H.: Consistency regularization for
  generative adversarial networks. In: International Conference on Learning
  Representations (2020)

\bibitem{zhang2018mixup}
Zhang, H., Cisse, M., Dauphin, Y.N., Lopez-Paz, D.: mixup: Beyond empirical
  risk minimization. International Conference on Learning Representations
  (2018), \url{https://openreview.net/forum?id=r1Ddp1-Rb}

\bibitem{zhao2020diffaugment}
Zhao, S., Liu, Z., Lin, J., Zhu, J.Y., Han, S.: Differentiable augmentation for
  data-efficient gan training. In: Conference on Neural Information Processing
  Systems (NeurIPS) (2020)

\end{thebibliography}

\clearpage

\end{document}